\documentclass[letterpaper, 10 pt, conference]{ieeeconf}  
\IEEEoverridecommandlockouts   

\usepackage{amsfonts}
\usepackage{amsthm}
\makeatletter
\def\th@plain{%
  \thm@notefont{}
  \itshape 
}

\makeatother
\usepackage{mathtools}      
\usepackage{amssymb}

\usepackage[linesnumbered,algoruled,boxed,vlined, noend]{algorithm2e}
\usepackage{dsfont}
\usepackage{comment}
\usepackage{xcolor}
\usepackage{graphicx}
\usepackage{cite}
\usepackage[font=small,labelfont=bf]{caption}
\usepackage{subcaption}
\usepackage{diagbox}
\usepackage{algorithmic}
\usepackage{wrapfig}
\let\labelindent\relax
\usepackage{enumitem}

\usepackage{overpic}        

\usepackage[textsize=footnotesize]{todonotes}
\usepackage{xcolor}
\usepackage{xargs} 
\newcommandx{\kg}[2][1=]{\todo[linecolor=red,
			backgroundcolor=red!10,bordercolor=red,#1]{#2}}
\newcommandx{\jy}[2][1=]{\todo[linecolor=green,
			backgroundcolor=green!10,bordercolor=green,#1]{JY: #2}}
\newcommandx{\sw}[2][1=]{\todo[linecolor=blue,
			backgroundcolor=blue!10,bordercolor=blue,#1]{SW: #2}}

\usepackage[normalem]{ulem}

\newtheoremstyle{mystyle}
  {}
  {}
  {\itshape}
  {}
  {\bfseries}
  {.}
  { }
  {}

\theoremstyle{mystyle}

\newtheorem{theorem}{Theorem}[section]
\newtheorem{proposition}{Proposition}[section]

\theoremstyle{definition}

\theoremstyle{remark}

\SetKwProg{Fn}{Function}{}{}
\SetKwComment{Comment}{$\triangleright$\ }{}

\title{
Barrier Forming: Separating Polygonal Sets with Minimum Number of Lines 
}

\author{ Si Wei Feng \and Jingjin Yu
\thanks{ S. W. Feng, and J. Yu are with the Department of Computer Science, Rutgers, the State University of New Jersey, Piscataway, NJ, 
 USA. E-Mails: \{{\tt siwei.feng,  jingjin.yu}\}\hspace*{.25em}
 @ \hspace*{.25em}rutgers.edu.
}
\thanks{
This work is supported by NSF awards IIS-1845888 and IIS-2132972.
}
}

\begin{document}
\maketitle
\thispagestyle{empty}
\pagestyle{empty}

\begin{abstract}
In this work, we carry out structural and algorithmic studies of a problem of barrier forming: selecting the minimum number of straight line segments (barriers) that separate several sets of mutually disjoint objects in the plane. 
The problem models the optimal placement of line sensors (e.g., \\
\noindent infrared 
laser beams) for isolating many types of regions in a \\
\noindent pair-wise manner for practical purposes (e.g., guarding against intrusions). 
The problem is NP-hard even if we want to find the minimum number of lines to separate two sets of points in the plane. 
Under the umbrella problem of barrier forming with minimum number of line segments, 
three settings are examined: barrier forming for point sets, point sets with polygonal obstacles, polygonal sets with polygonal obstacles.
We describe methods for computing the optimal solution for the first two settings with the assistance of mathematical programming, and provide a $2$-OPT solution for the third.
We demonstrate the effectiveness of our methods through extensive simulations.
\end{abstract}

\section{Introduction}

Consider the scenario where one or more rogue agents (e.g., criminals) may be hiding in several isolated regions in a 2D workspace. To prevent them from potentially escaping from these regions to other nearby vulnerable regions, we may wish to set up line-of-sight sensors to detect if rogue agents attempt to escape. For the setup, a natural question one may ask is: what is the minimum number of line segments that are needed to form the desired barrier? The same setting finds many other practical applications, for example, for the identical setting, we may use the deployed sensors to track the movement of agents between different set of regions, e.g., understanding the flow of people between residential areas to commercial areas, which can benefit large-scale decision making, e.g., to help properly allocating resources for improving the city infrastructure. 
Alternatively, the computed line segments can serve as patrolling routes for autonomous agents (robots or humans) for actively monitoring intrusions, where the agents can always keep tracking events along the segments for which they are responsible.

%
Motivated by the above stated scenarios and inspired by earlier research in robotics on barrier forming \cite{kloder2007barrier,kloder2008partial}, i.e., erecting barriers for separating regions of interest, in this paper, we examine the variation of finding the minimum number of straight line segments for isolating multiple sets of points of polygons in a two-dimensional workspace. Fig~\ref{fig:ex}[left] illustrates an instance where three colored polygonal sets are to be separated from each other and the grey polygons are obstacles. Fig~\ref{fig:ex}[right] shows a possible solution which is fairly non-trivial 
to obtain. 
%

\begin{figure}[t]
    \centering
    \includegraphics[width=\columnwidth]{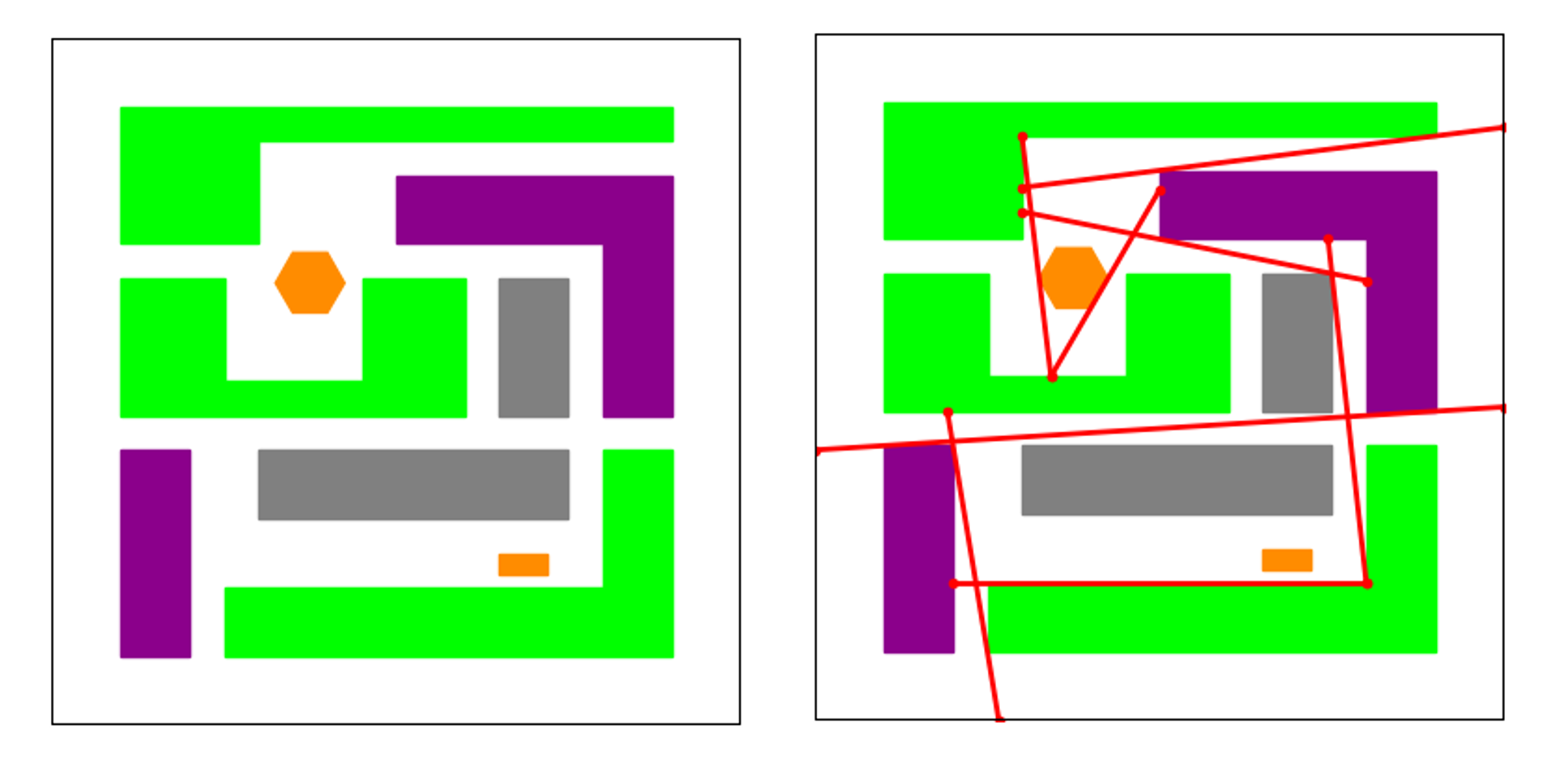}
    \caption{Example of barrier forming for separating three sets of complex polygonal shapes. Different colors represent different sets with the grey ones representing obstacles. The red line segments are the barriers computed by our algorithm.}
    \label{fig:ex}
    \vspace{-0.2in}
\end{figure}
As a summary of this work and its contributions, we study three settings in using line segments to separate sets of disjoint geometric shapes in a 2D workspace: (1) separating point sets, (2) separating point sets among polygonal obstacles, and (3) separating polygonal sets among polygonal obstacles. 
Whereas all three settings are NP-hard to optimally solve, we derive an effective method for computing optimal solutions for the first two settings, capable of handling tens of objects (points and/or polygons) partitioned into multiple sets. 
The  method first systematically computes candidate barrier set containing a minimum separating barrier, and then builds a novel integer programming model for finding the minimum barrier. 
Following a similar approach, we also develop a method that computes solutions for the third setting that is proven to be at least $2$-optimal. 
We provide theoretical analysis that shed some light on why the setting involving polygonal sets is more difficult to solve.
Extensive simulation study corroborates the effectiveness of our algorithms. 


\noindent\textbf{Related Work}.
Our investigation of barrier forming has its origins from several research areas.
In the study of pursuit evasion or more generally differential games \cite{ho1965differential,isaacs1999differential,hajek2008pursuit,tovar2009sensor,simov2000pursuit,guibas1997visibility,kameda2006online,kirousis1986searching, wen2018localization, sachs2004visibility,lau2005optimal, yu2011shadow}, 
such scenarios often happen where several agents are tasked to search an environment for hidden rogue agents or to protect critical regions from outsider invaders, which amount to creating and maintaining static or dynamic barriers of some form.  
Among the approaches taken in finding solutions for these problems, some resort to discretize the environment into graph representations and conduct search over them \cite{kirousis1986searching, sachs2004visibility}, while others adopt probabilistic reasoning \cite{lau2005optimal, yu2011shadow}. 
In particular, Tovar et al. \cite{tovar2009sensor} studied a problem that examines how to untangle sensor beam crossings to reason about the state of a robot, and use the insight to build algorithms for driving a robot to a desired state. In a sense, their study can be viewed as studying a problem that is a dual of the problem that we study here. 


Barrier forming problems can also be seen as a type of sensor coverage problem. 
A series of study of perimeter guarding problems aim at covering the boundaries of some regions to protect them from the outside \cite{FenHanGaoYuRSS19, FenYuRSS20}, which share similar motivation. 
The barrier forming problem has more flexible solutions by not fixing the specific boundaries to secure for the critical regions,
which can be more general and closer to reality.
In \cite{kloder2007barrier}, the authors solved a barrier coverage problem, the objective of which is to minimize the total length of the barrier among two groups of polygonal objects, optimally under non-trivial environment. To tackle the proposed problem, a novel and efficient network-flow based method is applied. 
In \cite{abrahamsen2020geometric}, the work is extended to multiple groups of objects.
Their following work \cite{kloder2008thesis} includes a similar problem formation to the problem studied in this paper, where the objective is minimizing the number of fixed-length line segments used. 
However, the proposed solution method in \cite{kloder2008thesis} uses Tarski sentence \cite{tarski1949decision} which, to our knowledge, does not have very effective method to solve.

This work is closely related to several problems in computational geometry, especially point set shattering which seeks a complete separation among a single set of points \cite{freimer1991complexity, har2020separating}. 
Bichromatic point sets or polyherdral separation problem uses wedges, axis-aligned lines, chords, parallel lines, or circles to separate two sets of points or polyhedra \cite{devillers2001separating,armaselu2017geometric,boissonnat2001circular, demaine2005separating}. 
Work on these problems often put on specific constraints such as working with convex objects or simple polygons without holes, adopting non-crossing or parallel lines as the separator, and so on. 

\noindent \textbf{Paper Structure}.
The rest of the paper starts with formulating the barrier formation problem and introducing the three variants studied in Section~\ref{sec:preliminary}. 
Then, we describe the structure analysis of the problem in Section~\ref{sec:structure}, which will base the algorithm
proposed in Section~\ref{sec:algorithm}. Lastly, in Section~\ref{sec:evaluation} we evaluate the algorithm on four different scenarios. 


\label{sec:intro}

\section{Preliminaries}\label{sec:preliminary}
\subsection{Barrier Forming with Minimum Number of Line Segments}

\begin{figure}[ht]
    \centering
    \vspace{.05in}
    \includegraphics[width = .35\textwidth]{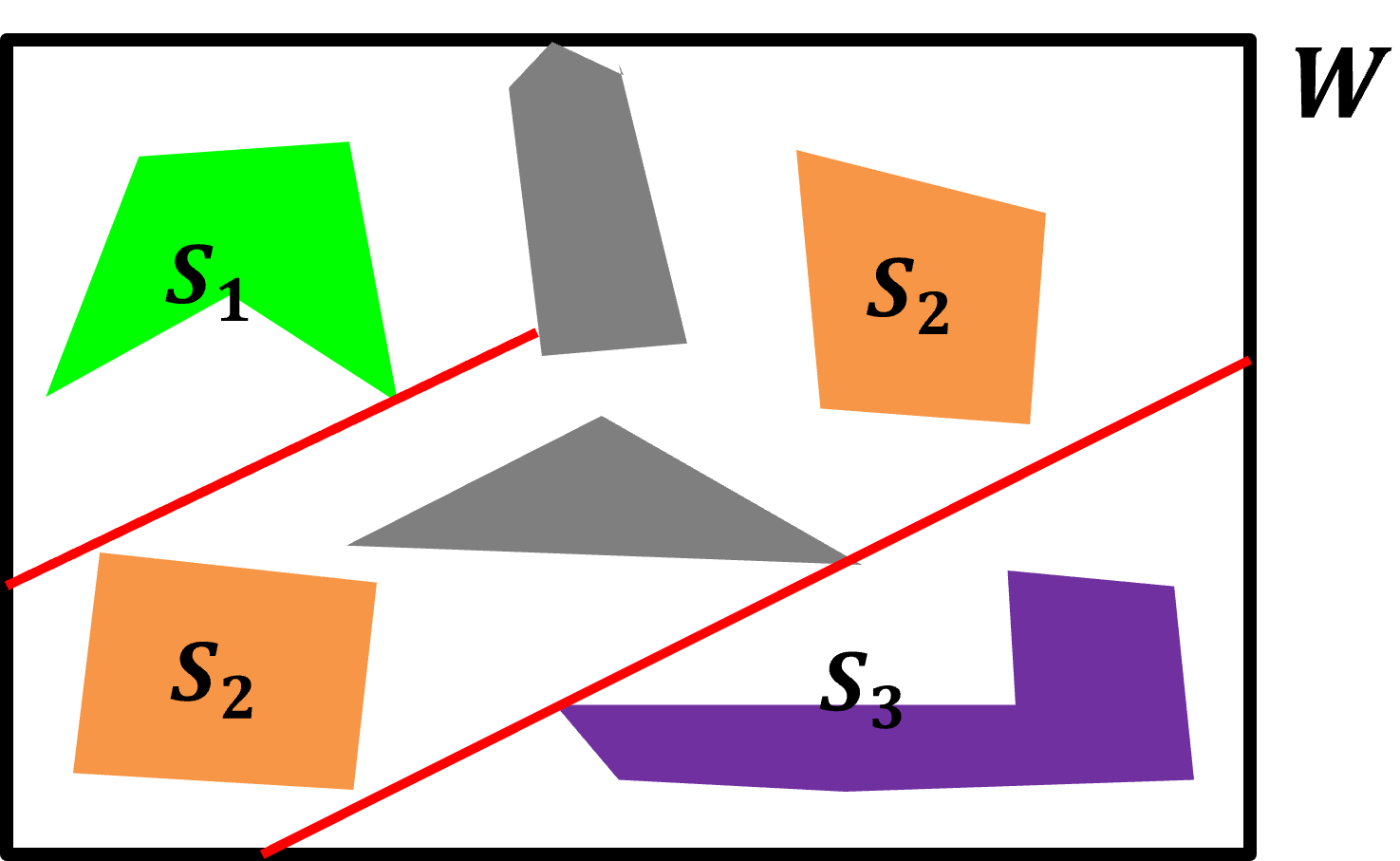}
    \vspace{.05in}
    \caption{Illustration of a barrier forming problem for the three sets of (colored) polygonal objects $S_{1\sim3}$ among (gray) obstacles. For the setting, $2$ straight lines are sufficient to cut off path connections between any pair of polygon regions from two different sets. In this work, our general goal is to build such barriers with the least number of line segments.}
    \label{fig:my_label}
\end{figure}

Let $\mathcal{W}$ be a simply connected polygonal workspace in $\mathbb R^2$. Consider $k$ sets of objects $S_1, \dots, S_k$ in $\mathcal W$, as well as a set of polygonal obstacles $\mathcal O$. 
We seek to use a set of straight line segments $L$ to separate $S_1, \dots, S_k$ from each other, such that for any two objects $o_1 \in S_i$ and $o_2 \in S_j, i \ne j$, any path between $o_1$ and $o_2$ in the free space of $\mathcal W$ must be ``cut off'' by one or more line segments from $L$. 
The line segments do not have length limit and can cross each other, but they cannot cross objects or obstacles in the workspace. We want to find the minimum number of line segments to complete the separation task. 

Under the general formulation of Barrier Forming, we examine three different variants with increasing difficulty:
\begin{enumerate} 
\item \emph{Barriers for point sets}, in which  $S_1, \dots, S_k$ as well as $O$ are sets of points. 
\item \emph{Barriers for point sets among polygonal obstacles}, in which  $S_1, \dots, S_k$ are sets of points but $O$ is a set of polygonal objects. 
\item \emph{Barriers for polygonal objects}, in which $S_1, \dots, S_k$ as well as $O$ contain polygonal objects. 
\end{enumerate}

\subsection{Computational Intractability}
Because the separation of even two sets of points in the plane, a special case of the first variant of our barrier forming problem, is computationally intractable \cite{demaine2005separating}, our first formulation is also NP-hard. From here, we can reduce from the first variant to other variants involving polygonal shapes by converting each point object to a sufficiently small polygon. Therefore, all three versions of the barrier forming problems studied in this paper are NP-hard. We omit the straightforward details.

\section{Structural Analysis}
Given that barrier forming problems studied in this work are NP-hard, a natural algorithmic choice for addressing the challenge is through exploring  mathematical programming.
To that end, a model must be built that selects from candidate barriers, which in turn requires the construction of a representative set of barrier candidates, a rather non-trivial task. 
The set of candidate barriers should satisfy two conflicting constraints: (1) it should contain a minimum set line segments that achieves the desired separation and (2) its size should not be too big that it will cripple the barrier selection process. 
Through careful structural analysis, we notice that the barriers to be considered can be limited to \emph{tangent} or \emph{bitangent} line segments. A tangent line segment, with respect to an object or an obstacle, is a line that passes through a vertex or an edge of the object/obstacle but does not intersect its interior. A bitangent is a line segment that is tangent to two objects and/or obstacles. 
This allows us to significantly reduce the number of candidates to be examined at the later selection stage.

\begin{theorem}\label{theorem:sin_tan}
For any $k$ sets of polygonal or point objects $S_1, \dots, S_k$ in the workspace $\mathcal W$, the set of line segments that are tangential to the objects and obstacles contains a set of minimum cardinality that separates $S_1, \dots, S_k$. 
\end{theorem}

\begin{proof}

Consider a set of line segments $L^*$ with minimum cardinality that separates $S_1,\dots, S_k$. 
Without loss of generality, we assume all line segments in $L^*$ do not end in the free space, i.e., each line segment in $L^*$ ends
at either object boundaries or workspace boundaries.
If some line segment in a minimum barrier is not tangent to any object vertex, denoted it as $\ell=OA$ (shown in Fig.~\ref{fig:proof}), we show that it can be replaced by a line segment that is tangent to some object vertex. 
Fix one end of $\ell$, $O$ in this case, and rotate $\ell$ around $O$ in both clockwise and counterclockwise directions until it hits some object vertex and becomes tangential to the object.
Denote the two line segments resulting from clockwise rotation and counterclockwise rotation as $\ell_1'=OB$ and $\ell'_2=OC$, respectively. 

We show $\ell$ can be replaced with $\ell_1'$ or $\ell_2'$. 
If this is not the case,
since replacing $\ell$ with $\ell_1'$ cannot make the separation work, there must be some point $P_1$ between $AB$ that is path connected to some point in the other class without crossing any line segments in $L^*$ when $\ell$ is replaced with $\ell'_1$. Denote the point as $D_1$ and the path as $path_1$. 
The same analysis goes for $\ell'_2$, that if $\ell$ cannot be replaced by $\ell_2$ then there is some point $P_2$ in $AC$ and path $path_2$ that connects $P_2$ to some other point $D_2$ in a different class and crosses segment $\ell$ but not $\ell_2'$. 
Since there are no objects or obstacles inside triangle $OCB$, we can assume the parts of $path_1$ and $path_2$ inside triangle $OCB$ are straight lines.
So, $path_1$ and $path_2$ must cross each other at some point. 
Denote the cross point as $Q\in path_1 \cap path_2$. 
Then, $path_1 = path_{11} (from\ P_1\ to\ Q) + path_{12} (from\ Q\ to\ D_1)$ and $path_2 = path_{12} (from\ P_2\ to\ Q) + path_{22} (from\ Q\ to\ D_2)$. 
Path $p_{11} + p_{22}$ connects $P_1$ to $D_2$, and $p_{21} + p_{12}$ connects $P_2$ to $D_1$, one of which must not cross $\ell$. 
This leads to a contradiction to the fact that the original line set $L^*$ separates the $k$ classes of objects.

Therefore, each non-tangent line segment in $L^*$ can be replaced with a tangent line segment. 
It will eventually result in a set of tangent barriers with minimum cardinality that separates $S_1,\dots,S_k$.
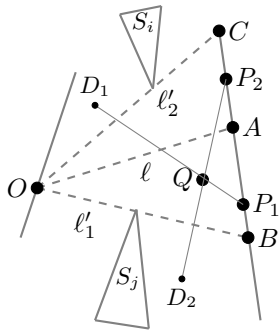
\begin{figure}[ht]
    \vspace{-2mm}
    \centering
\begin{tikzpicture}[scale = 1.1]
\draw[gray, thick] (0.1, 0) -- (0.7666, 2);
\draw[gray, thick] (3, -1) -- (2.5, 2.5);

\draw[gray, thick] (1, -1) -- (1.5, 0.34);
\draw[gray, thick] (1.65, -1.1) -- (1.5, 0.34);
\draw[gray, thick] (1.65, -1.1) -- (1, -1);

\draw[gray, thick] (1.3, 2.7) -- (1.7, 1.8);
\draw[gray, thick] (1.8, 2.8) -- (1.7, 1.8);
\draw[gray, thick] (1.8, 2.8) -- (1.3, 2.7);

\draw[gray, thick, dashed] (0.3, 0.6) -- (2.66666, 1.333);

\draw[gray, thick, dashed] (0.3, 0.6) -- (2.5, 2.5);
\draw[gray, thick, dashed] (0.3, 0.6) -- (2.85, 0);

\node[text width=1cm] at (2, 0.8) {$\ell$};
\node[text width=1cm] at (2.2, 1.63) {$\ell_2'$};
\node[text width=1cm] at (1.2, 0.15) {$\ell_1'$};

\filldraw[black] (0.3, 0.6) circle (2pt) node[anchor=east] {$O$};
\filldraw[black] (2.666, 1.333) circle (2pt) node[anchor=west] {$A$};
\filldraw[black] (2.5, 2.5) circle (2pt) node[anchor=west] {$C$};
\filldraw[black] (2.85, 0) circle (2pt) node[anchor=west] {$B$};

\filldraw[black] (2.79, 0.4) circle (2pt) node[anchor=west] {$P_1$};
\filldraw[black] (2.5833, 1.917) circle (2pt) node[anchor=west] {$P_2$};

\filldraw[black] (2.3, 0.7) circle (2pt) node[anchor=east] {$Q$};

\draw[gray] (2.79, 0.4) -- (1.0, 1.6);
\draw[gray] (2.5833, 1.917) -- (2.05, -0.5);
\filldraw[black] (1.0, 1.6) circle (1pt) node[anchor=south] {\small{$D_1$}};
\filldraw[black] (2.05, -0.5) circle (1pt) node[anchor=north] {\small $D_2$};

\node[text width=1cm] at (1.9, 2.6) {\small $S_i$};
\node[text width=1cm] at (1.7, -0.5) {\small $S_j$};

\end{tikzpicture}
    \caption{Rotating non-tangent barrier line segment $\ell$ in clockwise and counterclockwise directions around its endpoint $O$ until it becomes tangential to some objects.}
    \label{fig:proof}
    \vspace{-2mm}
\end{figure}
\end{proof}

Although we can limit the candidate barriers to line segments tangent to object vertices, there can
still be infinite number of candidates. 
One may consider using line segments that are bitangent to
object vertices, i.e. line segments crossing two object or obstacle vertices. If there are $n$ object/obstacle vertices, there can be at most $n^2$, i.e., a quadratic number of bitangents. 
Unfortunately, bitangent lines are insufficient to act as candidate barriers by themselves for polygonal objects. A counterexample in Fig.~\ref{fig:counter} shows that there is an instance where an optimal solution must contain line segments that are not bitangents. 
In this counterexample, we need separate the orange objects from the lime object. A minimum of three line segments are used, and it is not possible that all of them are bitangent, i.e.

\begin{proposition}
Bitangent line segments do not always contain optimal solution for the barrier forming problem for polygonal objects.
\end{proposition}

\begin{figure}[ht]
    \centering
    \vspace{-.2in}
    \includegraphics[width = .25\textwidth]{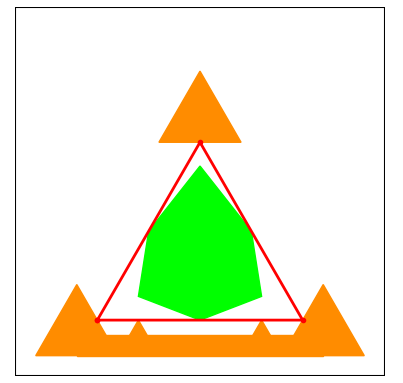}
    \vspace{0.0in}
    \caption{Counterexample that shows using only bitangent line segments cannot create the optimal solution}
    \label{fig:counter}
\end{figure}

Despite the caveat, for the first two formulations that deal with barrier forming for point sets, even with polygonal obstacles, bitangent line segments
always contain an optimal solution. More precisely, 
\begin{theorem}
For any $k$ sets of point objects $S_1, \dots, S_k$ in a workspace $\mathcal W$, there exists
a set of line segments with minimum cardinality that separates $S_1, \dots, S_k$, 
and only consists of bitangent line segments.
\end{theorem}

\begin{proof}
From Theorem~\ref{theorem:sin_tan}, we can see using single tangent line segments is always enough
for an optimal solution. 
Now we turn an optimal solution, $L^*$, with only tangent line segments, into
a solution with only bitangent line segments while still maintaining the same number of barriers. 

For a tangent line segment $\ell=AB\in L^*$ with tangent point $O$ (shown in Fig.~\ref{fig:proof_bi}), and if $O$ is a point object, assume it is beneath $\ell$,
rotate $\ell$ clockwise around $O$ until it hits a point object or an obstacle vertex.
Denote the resulting line segment as $\ell'$, and replace $\ell$ with $\ell'$.
Since the objects are point objects, so $BB'$ and $AA'$ must belong to obstacles or workspace boundary, 
and thus there is no object point inside $OAA'$ or $OBB'$. 
Therefore, the replacement won't result in
any path connecting objects in different classes.
If this is not the case, then there will be some path connecting two object points in different classes that crosses $\ell$ but does not cross $\ell'$ or other barriers. 
Since the triangle areas $OAA'$ and $OBB'$ are empty, that path must enter region $OAA'$ or $OBB'$ and leave them from $\ell$. Then, that part of the path could be replaced with a straight line segment parallel to $\ell$ which prevents it from crossing $\ell$. This contradicts the assumption that $L^*$ prevents all connections between objects in different classes. 

Continuing the replacement until all line segments are bitangent will result in an optimal solution with only bitangent line segments.

\begin{figure}[ht]
    \centering
\begin{tikzpicture}
\draw[gray, thick] (2.5, -1) -- (1.75, 2);
\draw[gray, thick] (-1, -1) -- (-1, 1.3);

\draw[gray, thick] (0.7, -1.5) -- (1, -0.14);
\draw[gray, thick] (1.2, -1.6) -- (1, -0.14);


%
\draw[gray, thick, dashed] (-1., -0.5) -- (2., 1.);

\draw[gray, thick, dashed] (-1., 0.1) -- (2.35, -0.3);


\node[text width=1cm] at (1.4, 0.3) 
                        {$\ell$};
\node[text width=1cm] at (1.8, -0.6) 
                        {$\ell'$};

\filldraw[black] (0.0, -0.1) circle (2pt) node[anchor=north] {$O$};
\filldraw[black] (2, 1.) circle (2pt) node[anchor=west] {$A$};
\filldraw[black] (-1, -0.5) circle (2pt) node[anchor=east] {$B$};

\filldraw[black] (2.35, -0.3) circle (2pt) node[anchor=west] {$A'$};
\filldraw[black] (-1, 0.1) circle (2pt) node[anchor=east] {$B'$};


\end{tikzpicture}
    \caption{Rotating a single tangent barrier line segment $\ell$ around its tangent point $O$ clockwise until it becomes bitangent.}
    \label{fig:proof_bi}
\end{figure}
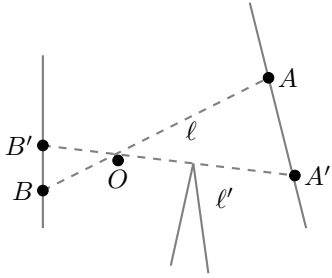
\end{proof}

For separating polygonal objects, although using bitangent line segments cannot guarantee an optimal solution that uses minimum number of line segments, they can still ensure that solutions limited to bitangents are at least $2$-optimal.
\begin{proposition}
For any $k$ sets of polygonal objects $S_1, \dots, S_k$ in the workspace $\mathcal W$, there exists a set of line segments with cardinality at most twice the minimum cardinality, that separates $S_1, \dots, S_k$, and only consists of line segments that are bitangent to object or obstacle vertices. 
\end{proposition}
\begin{proof}

\begin{figure}[ht]
    \centering
\begin{tikzpicture}

\draw[gray, thick] (2.5, -1) -- (1.75, 2);
\draw[gray, thick] (-1.5, -1) -- (-1, 1.5);
\draw[gray, thick] (0, 0) -- (-0.3, 1);
\draw[gray, thick] (0, 0) -- (.2, 1.5);

\draw[gray, thick] (-0.8, -1.2) -- (-0.5, -0.25);
\draw[gray, thick] (-0.45, -1.25) -- (-0.5, -0.25);

\draw[gray, thick] (0.6, -0.9) -- (0.9, -0.22);
\draw[gray, thick] (1.1, -1.2) -- (0.9, -0.22);

\draw[gray, thick, dashed] (-1.45, -0.75) -- (2., 1.);
\draw[gray, thick, dashed] (-1.3, -0.0) -- (2.2, 0);

\draw[gray, thick, dashed] (-1.2, 0.32) -- (2.45, -0.6);

\node[text width=1cm] at (2.2, -0.2) 
                        {$\ell$};

\node[text width=1cm] at (1.8, -0.6) 
                        {$\ell'_1$};

\node[text width=1cm] at (1.8, 0.4) 
                        {$\ell'_2$};

\filldraw[black] (0.0, -0.0) circle (2pt) node[anchor=north] {$O$};

\end{tikzpicture}
    \caption{Rotating tangent barrier line segment $\ell$ both clockwise and counterclockwise around its tangent point $O$ until it becomes bitangent.}
    \label{fig:proof_bi_2opt}
\end{figure}
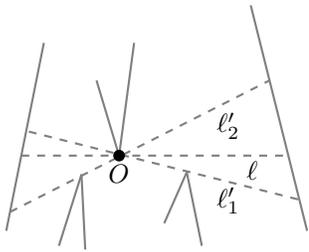

Starting from an optimal solution $L^*$ with only tangent line segments,
we will replace each tangent line segment with two bitangent line segments.

Rotate each non-bitangent line segment $\ell\in L^*$ around its tangent point $O$ in clockwise or counterclockwise directions until the line segment become bitangent, as illustrated in Fig.~\ref{fig:proof_bi_2opt}. 
Since any path connecting two objects in different classes and is cut by barrier $\ell$ will still be cut by $\ell'_1$ and $\ell'_2$.
The replacement can still guarantee the separation among the object groups.
After replacing all barriers, we can obtain a 2-OPT solution with the number of line segments twice the minimum.
\end{proof}
\label{sec:structure}


\section{Fast Computation of High-Quality Solutions}\label{sec:algorithm}
In this section, we will apply the structural results of the barrier forming problem and provide effective method to tackle it.
First, we start with a general method for obtaining the optimal solution among a set of candidate barriers.
Then, based on different ways to generate the candidate barriers, two approaches are given while one is based on 
using bitangent line segments and the other is based on sampling.

\subsection{Optimal Solution for Given Line Separator Candidates}
\label{sec:algo:ilp}
In the barrier forming problem, if the candidate barriers are available as a finite set, we can tackle the problem with integer programming (IP). 
To solve it, we first perform a decomposition of the workspace using the candidate barriers, which partitions the workspace into cells whose edges are part of some candidate barriers. 
Denote  $N$ as the number of candidate barrier line segments, and $M$ as the number of 
cells dissected using the candidate barriers. Fig.~\ref{fig:ilp_example} shows an example of dissecting the workspace into six cells with three candidate barriers.
Then, we can start to construct an IP model to solve the problem of minimizing the number of selected barriers. 
First, we use $\lceil\log k\rceil$ binary variables for each cell, 
resulting in $M\cdot\lceil \log k \rceil$ such variables $c_{11}, \dots, c_{M\lceil\log k\rceil}$. 
The value of $\overline{c_{i1}c_{i2}\dots c_{i\lceil \log k\rceil}}$ will represent the class of cell $i$. 
Thus, if there is an object in cell $i$, $\overline{c_{i1}c_{i2}\dots c_{i\lceil \log k\rceil}}$ should have a fixed value according to the class of the object.
A binary variable for each candidate line segment is used to indicate whether that line candidate is selected,
resulting in $N$ such variables: $\ell_1, \dots, \ell_N$. 

\begin{figure}
    \centering
    \begin{overpic}[width=0.3\textwidth]{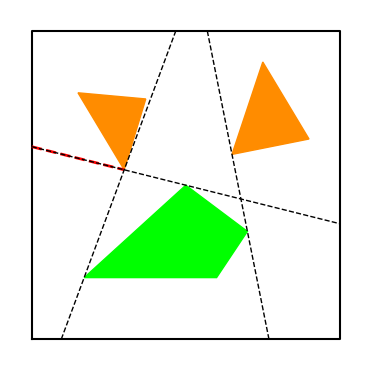}
    \put(0, 60) {$\ell_1$}
    \put(40, 95) {$\ell_2$}
    \put(55, 95) {$\ell_3$}
    
    \put(15, 80) {$c_1$}
    \put(50, 75) {$c_2$}
    \put(80, 80) {$c_3$}
    \put(15, 30) {$c_4$}
    \put(50, 15) {$c_5$}
    \put(80, 30) {$c_6$}
    \end{overpic}
    \caption{In this example, we aim to separate two groups of objects with the given candidate barriers. The workspace is decomposed into 6 cells by the 3 candidate barriers. As an example of constraint setup, the pair of adjacent cells $c_1$ and $c_4$ create a constraint of
    $\ell_1\geq c_1 \oplus c_4$, which is equivalent to $\ell_1\geq c_1 - c_4 \wedge \ell_1\geq c_4 -c_1$. (Since there is only $k=2$ classes of objects and $\log k =1$ here, the second index of $c_{i*}$ is eliminated.)} 
    \label{fig:ilp_example}
    \vspace{-0.2in}
\end{figure}
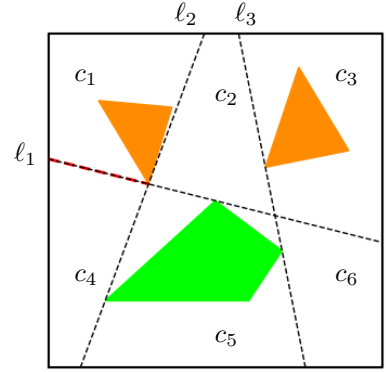

Between each pair of adjacent cells $i$ and $ j$ (let the candidate line segment between them correspond to $\ell_\sigma$), we have ($\oplus$ denotes ``exclusive or'')
\begin{equation}
\begin{split}
    &\ell_\sigma \geq c_{i\tau} \oplus c_{j\tau}  \Leftrightarrow \ell_\sigma \geq c_{i\tau} - c_{j\tau} \wedge \ell_\sigma\geq c_{j\tau} - c_{i\tau}, \\
    &\text{  for each adjacent cell $i$, $j$, and $1\leq \tau \leq \lceil \log k \rceil$},
\end{split}
\end{equation}
which means if two adjacent cells belong to different classes, then the line segment candidate between them must be selected. An example of the constraint setup is illustrated in Fig.~\ref{fig:ilp_example}.
Naturally, we have fixed $c_{i1},\dots,c_{i\lceil \log k\rceil}$ if ceil $i$ contains an object to be separated, and the value of $\overline{c_{i1}c_{i2}\dots c_{i\lceil \log k\rceil}}$ is set to be the same as its class index: $0\sim k-1$.
The objective is set to minimize the total number of line candidates selected, i.e., $\min \sum_{i=1}^{M} l_i$.

\subsection{Near-optimal Solution Using Bitangent Lines}
As the results in Section~\ref{sec:structure} show, using bitangent line segments 
can always provide optimal solution for the problem of barrier forming for point objects, 
and at least 2-OPT
solution for separating polygonal objects. 
Since the number of bitangent line segments is at 
most quadratic to the number of object or obstacle vertices, 
we can enumerate them, and apply IP method in Section~\ref{sec:algo:ilp} to find a solution. Fig.~\ref{fig:barrier_candidates} illustrates the candidate barriers constructed for point objects and polygonal objects.
It is worth noting that for enumerating bitangent barriers for point objects, the side of the point to the barrier also matters. For example, a pair of point objects will create $4$ bitangent barrier candidates as there are 4 different possible cases depending on the how the corresponding objects are placed with respect to the line. 

\begin{figure}[ht]
    \centering
    \includegraphics[trim=80 20 80 20,clip, width = .24\textwidth]{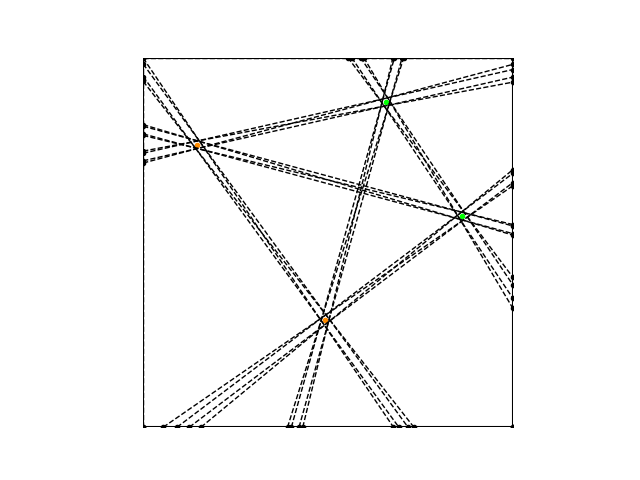}
    \hspace{-.1in}
    \includegraphics[trim=80 20 80 20,clip, width = .24\textwidth]{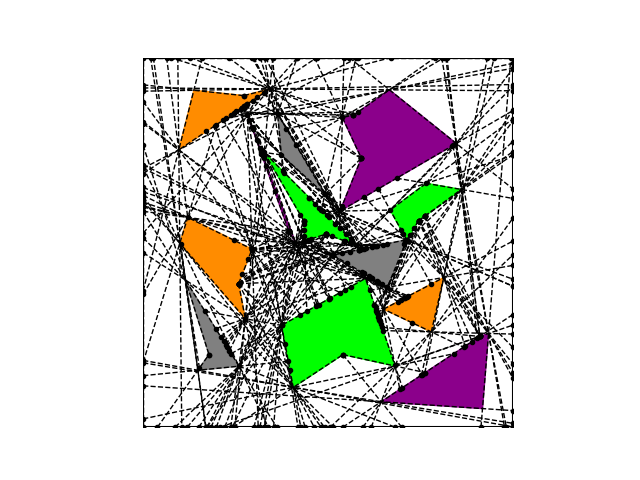}
    \caption{Illustration of bitangent barrier candidates. The left picture shows the bitangent
    candidates for $2$ point sets, each with $2$ points. In this case, a pair of points will create $4$ candidates. We note that we made the points non-zero-dimensional for visibility purposes. The right picture shows the bitangent candidates for $12$ polygonal objects in four sets.}
    \label{fig:barrier_candidates}
\end{figure}

\subsection{Sampling-based Resolution Complete Algorithm}
Although using bitangent line segments works well in most of the cases, it unfortunately cannot provide an optimal guarantee for the barriers formed when dealing with polygonal objects.
However, theorem~\ref{theorem:sin_tan} provides a good starting point for sampling line segments: we may limit candidate barrier sets to single tangents, i.e., we sample line segments passing through each object vertex in a radial manner. Hence, if we gradually increase the sampling resolution around each object and obstacle vertex and use the sampled line segments as candidate barriers, we can guarantee the asymptotic optimality of the resulting solution.





\section{Experimental Evaluation}\label{sec:evaluation}
In this section, we describe our experimental study of the method proposed in the paper.
Four settings were used to account for the three variants of the barrier forming, the instances and solution examples of which are shown in Fig.~\ref{fig:exp}. 
The experiments were carried out on a Hexa-Core processor with 16 GiB memory, and Gurobi
\cite{optimization2019gurobi} was used as the Integer Programming solver.
In the polygonal object instances generation, random polygons were generated by computing traveling salesperson tour (TSP) tours of random point sets each consisting of $3$ to $6$ vertices, which we found to be effective in generating sensible looking polygons that are not necessarily convex.
When a problem setup contains obstacles, the number of obstacles in the environment is set to be the same as the number of objects for each object set.

\begin{figure*}[ht]
    \centering
    
    \includegraphics[trim=80 20 80 20,clip, width = .24\textwidth]{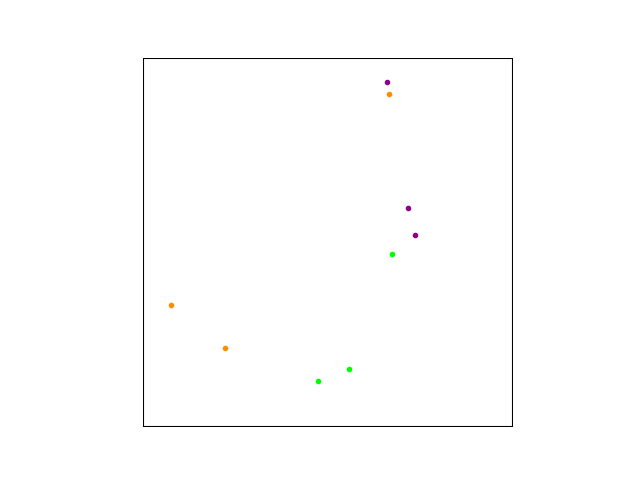}
    \hspace{-.1in}
    \includegraphics[trim=80 20 80 20,clip, width = .24\textwidth]{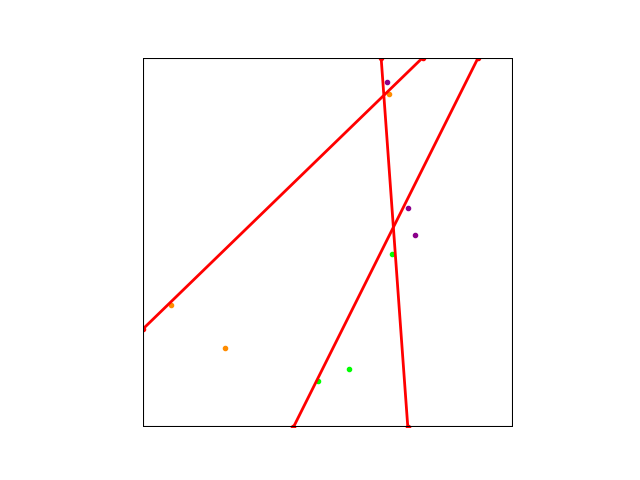}
    \includegraphics[trim=80 20 80 20,clip, width = .24\textwidth]{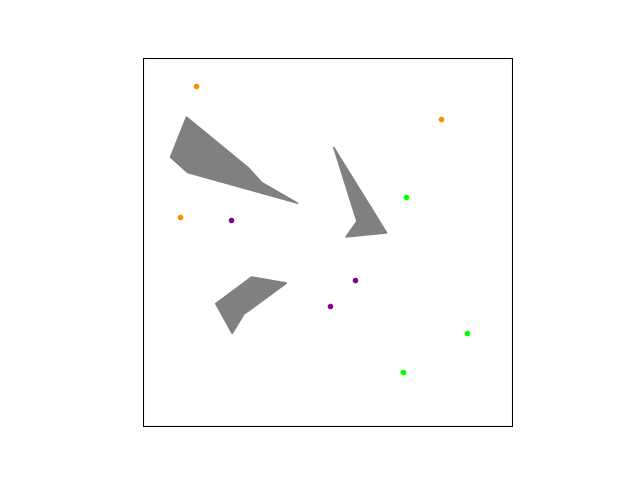}
    \hspace{-.1in}
    \includegraphics[trim=80 20 80 20,clip, width = .24\textwidth]{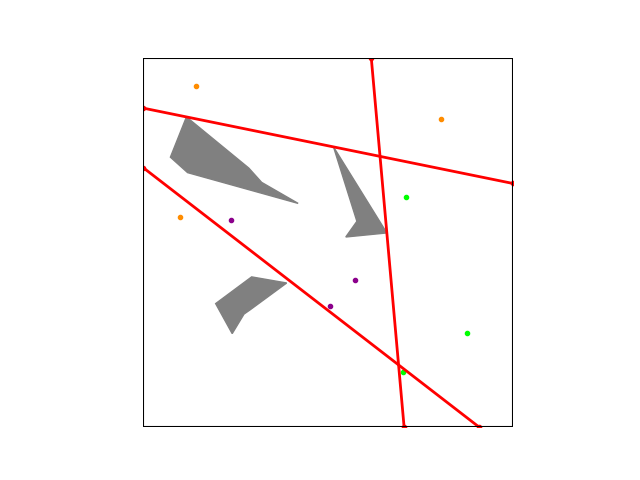}
    \includegraphics[trim=80 20 80 20,clip, width = .24\textwidth]{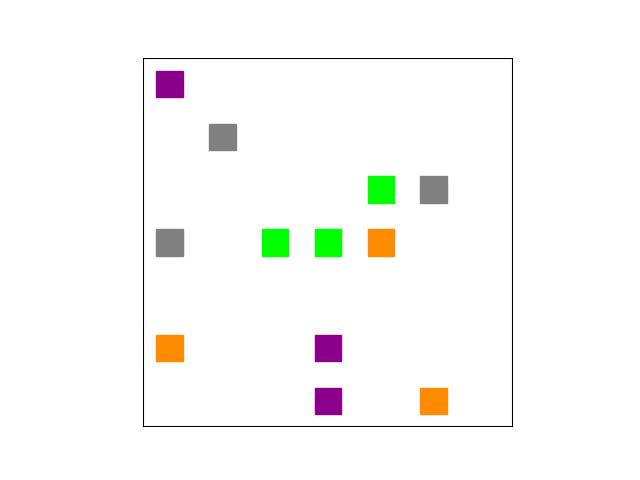}
    \hspace{-.1in}
    \includegraphics[trim=80 20 80 20,clip, width = .24\textwidth]{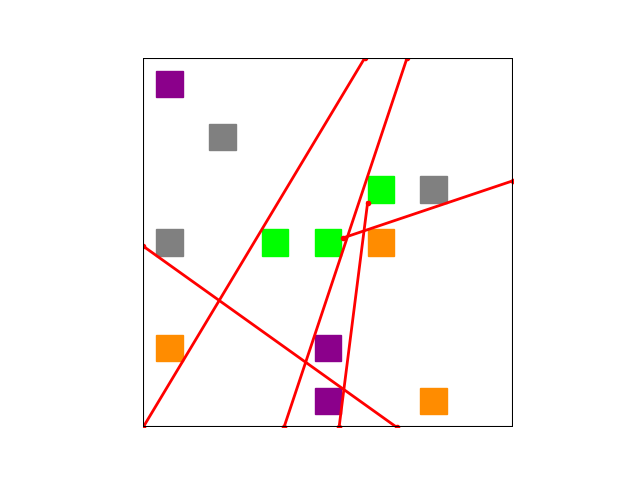}
    \includegraphics[trim=80 20 80 20,clip, width = .24\textwidth]{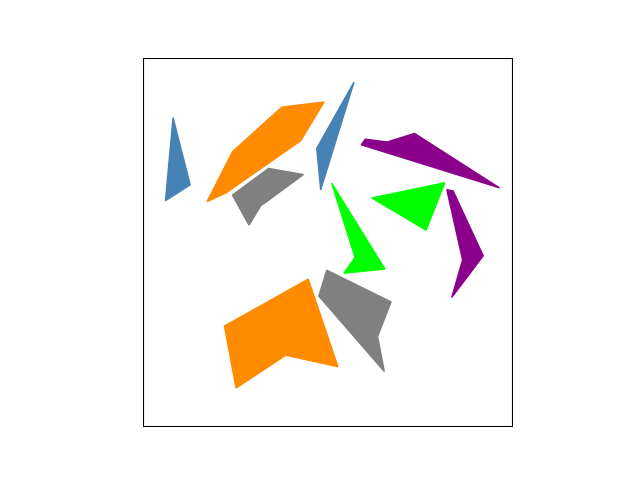}
    \hspace{-.1in}
    \begin{overpic}[trim=80 20 80 20,clip, width = .24\textwidth]{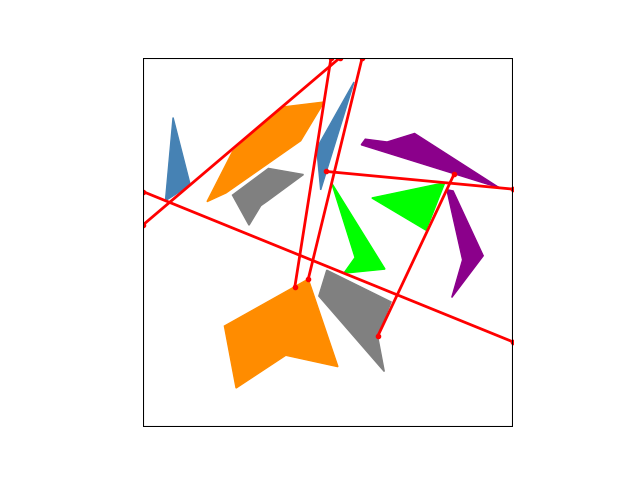}
    \put(-202, 98) {(a)}
    \put(-3, 98) {(b)}
    \put(-202, 0) {(c)}
    \put(-3, 0) {(d)}
    
    \end{overpic}
    \caption{Illustration of the four types of instances used in our  experimental evaluation. (a) Barrier forming for point sets. (b) Barrier forming to separate point sets from polygonal obstacles. (c) Barrier forming to separate uniform square-shaped objects among uniform square-shaped obstacles. (d) Barrier forming to separate random polygonal objects among random polygonal obstacles.}
    \label{fig:exp}
\end{figure*}

\subsection{Separating Sets of Points}
The first type of instances aims at forming barriers among randomly generated point sets, shown in Fig.~\ref{fig:exp}(a). The number of object sets to separate from each other range from $2$ to $4$, and the number of objects in each set range from $1$ to $6$. Each entry in the Table ~\ref{tab:expr_1} is the result of average computation times over 10 instances. From the result, we observe that the IP based method is fairly effective in separating two sets of objects with the presence of obstacles. The method scales to about $10$ point objects plus obstacles. The number of line segments in the optimal barrier are generally small, e.g., $3$-$10$.

\begin{table}[ht]
    \centering
    \begin{tabular}{|c|c|c|c|c|c|c|}\hline
        \#Sets &  1 & 2 & 3 & 4& 5& 6\\\hline
 2& 0.005 & 0.011 & 0.083 & 0.419 & 2.010 & 15.887 \\\hline
 3& 0.013 & 0.316 & 12.773 & 962.883 & - & - \\\hline
 4& 0.051 & 14.641 &  -& - &  -&- \\\hline
\end{tabular}
    \caption{Running time in seconds for Expr. 1 where all objects and obstacles are points (Fig.~\ref{fig:exp}(a)). ``-'' denotes the result cannot be computed in 1h on average (the same is true for other tables). 
    The column index means the number of objects in each set, and the row index means the number of object sets.
    The number of obstacles is set to be the same as the number of objects for each set. These also apply to the following tables.
    }
    \label{tab:expr_1}
\end{table}

 The second type of instances generates barriers for randomly generated 
 point sets with the existence of polygonal obstacles, shown in Fig.~\ref{fig:exp}(b). 
 The other specification is the same as Expr. 1, and the number of obstacles for each experiment is set to be the same as the number of objects for each set. 
 The resulting time cost, shown in Table~\ref{tab:expr_2}, is similar to Expr. 1 despite the existence of obstacles.
 Similarly, we observe decent performance when it comes to separating two sets of objects among obstacles. 
 The introduction of polygonal obstacles does not cause performance degradation.

\begin{table}[ht]
    \centering
    \begin{tabular}{|c|c|c|c|c|c|c|}\hline
        \#Set &  1 & 2 & 3 & 4& 5& 6\\\hline
2& 0.009 & 0.061 & 0.480 & 5.899 & 3.433 & 21.121\\\hline
 3& 0.044 & 3.287 & 71.346 & 320.955 & - & -\\\hline
 4& 0.249 & 13.801 & - & - & - & -\\\hline
    \end{tabular}
    \caption{Running time in seconds for Expr. 2 where objects to be separated are points and obstacles are randomly generated polygons (Fig.~\ref{fig:exp}(b)). 
    }
    \label{tab:expr_2}
    \vspace{-2mm}
\end{table}
 
\subsection{Separating Sets of Polygonal Shapes}
The third set of experiments uses randomly placed squares as obstacles and objects, shown in Fig.~\ref{fig:exp}(c). The squares are sampled from a $7\times7$ grid. Each square is half the scale of a grid cell and is positioned at the center of a grid cell. The running time for this case turns out to be the greatest among all $4$ experiments. This is due to the rectlinear nature of the instance, which creates many small cells that are difficult to process. 

\begin{table}[ht]
    \centering
    \begin{tabular}{|c|c|c|c|c|c|c|}\hline
         \#Set &  1 & 2 & 3 & 4& 5& 6\\\hline
 2 & 0.010 & 0.065 & 0.652 & 31.536 & 575.933 & 1259.653\\\hline
 3 & 0.065 & 10.608 & 337.050 & - & - & -\\\hline
 4 & 0.235 & 124.963 & - & - & - & -\\\hline
    \end{tabular}
    \caption{Running time in seconds for Expr. 3 with square-shaped objects and obstacles (Fig.~\ref{fig:exp}(c)). ``-'' denotes the result cannot be computed in 1h on average. 
    }
    \label{tab:expr_3}
    \vspace{-2mm}
\end{table}
 
The last type of instances uses random polygons with $3\sim 6$ vertices
as objects and obstacles, shown in Fig.~\ref{fig:exp}(d). 
Counter intuitively, these experiments turn out to have the least time cost among the $4$ experiments despite the most complex environment; we see that even for four different sets of objects where each set contains six objects, the problem can be solved very quickly. 

In the end, the running time of the algorithm provided is more dependent on the number of cells and candidate barrier line segments. 
When objects and obstacles are more densely packed in the experiment, 
there will be less barrier candidates and cells due to collisions between objects and the candidate line segments.
While in a sparse environment or even with just point objects, there will be more barrier candidates and cells.
This explains the reduced time cost in a more complex environment from the sparse settings.

\begin{table}[ht]
    \centering
    \begin{tabular}{|c|c|c|c|c|c|c|}\hline
         \#Set&  1 & 2 & 3 & 4& 5& 6\\\hline
2& 0.006 & 0.031 & 0.080 & 0.143 & 0.157 & 0.156 \\\hline
3& 0.031 & 0.234 & 0.994 & 1.271 & 1.980 & 1.187 \\\hline
4& 0.103 & 0.518 & 3.000 & 6.050 & 9.692 & 17.996 \\\hline 
    \end{tabular}
    \caption{Running time in seconds for Expr. 4 where both the objects to be separated and the obstacles are randomly generated polygons (Fig.~\ref{fig:exp}(d)).
    }
    \label{tab:expr_4}
    \vspace{-2mm}
\end{table}


\section{Conclusion}\label{sec:conclusion}

In this work, we have formulated a barrier forming problem with three
different variants, separating point sets, separating point sets with polygonal obstacles, and separating polygonal object sets with polygonal obstacles, for which we seek a solution containing the minimum number of straight line segment separators. 
We provided structural analysis for this NP-hard problem and describe
a novel barrier candidate proposal and filtering algorithm, which is optimal for the first two variants and at least 2-OPT for the third.
The method can effectively work with tens of objects including obstacles, as is illustrated in our four sets of experiments conducted. 
Nevertheless, finding more scalable algorithm to this problem, as well as solving barrier forming for polygons optimally, are of great interest.
The exploration of dynamic barrier forming with moving barriers is also an important direction to pursue.

\bibliography{bib/bib}
\bibliographystyle{IEEEtran}

\end{document}